\documentclass[aps,prresearch,twocolumn,superscriptaddress,longbibliography]{revtex4-2}

\usepackage{amsmath}
\usepackage{amssymb}
\usepackage{amsthm}
\usepackage{graphicx}
\usepackage{dcolumn}
\usepackage{bm}
\usepackage{hyperref}
\usepackage{xcolor}

\newtheorem{proposition}{Proposition}
\newtheorem{theorem}{Theorem}
\newtheorem{corollary}{Corollary}
\newtheorem{definition}{Definition}
\newtheorem*{remark}{Remark}

\newcommand{\R}{\mathbb{R}}
\newcommand{\C}{\mathbb{C}}

\newcommand{\Kc}{\mathcal{K}}

\newcommand{\diag}{\mathrm{diag}}

\begin{document}

\title{Non-normal spectral signatures of instability\\
in neural network training dynamics}

\author{Souvik Ghosh}
\affiliation{Department of Physics,
  National Sun Yat-sen University, Kaohsiung 80424, Taiwan}


\date{\today}

\begin{abstract}
Training instabilities in deep networks---loss spikes, oscillatory
convergence, and gradient pathologies---are empirically prevalent but
lack a rigorous operator-theoretic explanation.
We show that the linearized update operators for practically used
optimizers are generically non-normal: for Adam, non-normality is
controlled by the commutator $[H,M]$ between the Hessian and the
diagonal adaptive preconditioner, while for SGD with momentum it arises
from the augmented state-space structure of the update map.
Applying non-normal stability theory to these operators, we derive a
conservative pseudospectral precursor bound in which $\kappa(V)$ serves
as an early-warning indicator of transient amplification even when the
spectral radius remains below one, and we establish that exceptional
points of the update operator appear as the $\kappa(V)\to\infty$
limiting case of this framework.
Numerical experiments on two-layer networks confirm that the spectral
radius $\rho(J)$ provides no separation between stable and unstable
training phases while $\kappa(V)$ separates them by approximately one
order of magnitude, complementing the classical sharpness criterion with
a continuous severity measure of non-normal amplification.
These results establish non-Hermitian operator theory as a useful and
underexplored framework for neural network optimization stability,
offering a diagnostic language and proof-of-concept benchmark for
understanding adaptive optimization stability.
\end{abstract}

\maketitle

\section{Introduction}
\label{sec:intro}

Training instabilities---loss spikes, oscillatory convergence, and
gradient pathologies---are among the most practically consequential
phenomena in deep learning, yet they remain theoretically
undercharacterized.
The standard theoretical framework for optimization stability rests on
the eigenspectrum of the Hessian $H = \nabla^2 L(\theta)$: a loss
landscape is considered sharp when $\lambda_{\max}(H)$ is large, and a
training trajectory is predicted to be unstable when $\lambda_{\max}(H)$
exceeds $2/\eta$ for gradient descent with step size $\eta$~\cite{bottou2012}.
Recent work on the Edge of Stability~\cite{cohen2022,damian2023} has
refined this picture considerably, showing that the top Hessian
eigenvalue evolves dynamically and that sharpness-driven instabilities
recur progressively throughout training.

Despite this progress, the Hessian-based framework carries an implicit
assumption that has gone largely unexamined: since the Hessian $H$ is
symmetric, the stability analysis built on its eigenvalues inherits that
symmetry.
For vanilla gradient descent the update operator $J = I - \eta H$ is
symmetric and hence \emph{normal}---its eigenvectors are orthogonal and
its spectrum fully determines all dynamical behavior.
The spectral radius criterion $\rho(J) < 1$ is necessary and sufficient
for stability in this setting.

This assumption of normality fails for the update operators of
practically used optimizers.
For Adam-preconditioned gradient descent, the linearized update operator
takes the form $J = I - \eta M^{-1}H$, where $M$ is the diagonal
adaptive preconditioner.
We show that $J$ is normal if and only if $[H, M] = 0$---a condition
requiring the Hessian to be diagonal in the coordinate basis, which does
not hold generically for finite-width networks.
For SGD with momentum, the augmented state-space update operator on
$(\theta, v)$ is non-normal by virtue of its off-diagonal block
structure, irrespective of the Hessian.
In both cases the non-normality has a concrete consequence: transient
amplification of perturbations can occur and persist over many steps
even when all eigenvalues lie inside the stability boundary---a
phenomenon the spectral radius criterion is structurally incapable of
detecting~\cite{trefethen2005,trefethen1993}.

Non-normal operators and their associated stability pathologies are
well-studied in fluid mechanics~\cite{trefethen1993},
numerical analysis~\cite{trefethen2005}, and non-Hermitian condensed
matter physics~\cite{ashida2020}, but have not been systematically
applied to neural network optimization dynamics.
This paper makes that connection explicit.

We establish three main results.
First (Sec.~\ref{sec:operator}), we prove that the linearized update
operators for Adam and SGD with momentum are generically non-normal and
characterize the non-normality in terms of computable operator
quantities.
Second (Sec.~\ref{sec:pseudo}), we derive a conservative precursor bound
showing that the eigenvector condition number $\kappa(V)$ detects
impending transient amplification before the spectral radius crosses any
instability threshold, and connect this bound to the $\varepsilon$-pseudospectrum
and the Kreiss matrix theorem.
Third (Sec.~\ref{sec:numerics}), we validate these predictions
numerically: for SGD with momentum at a learning rate in the unstable
regime, $\kappa(V)$ separates stable and unstable training phases by
one order of magnitude while $\rho(J_\mathrm{aug})$ provides no
separation, consistent with Theorem~\ref{thm:precursor}.
The classical sharpness criterion $\lambda_{\max}(H) > 2/\eta$
provides a binary threshold signal consistent with the Edge of Stability
literature; $\kappa(V)$ provides complementary information as a
continuous severity measure within the unstable regime.
Exceptional points of the update operator appear naturally as the
$\kappa(V)\to\infty$ mathematical limit of the framework
(Sec.~\ref{sec:ep}), and implications for optimizer design are discussed
in Sec.~\ref{sec:discussion}.

\section{The Linearized Update Operator and Its Non-Normality}
\label{sec:operator}

\subsection{Optimizer Dynamics and Notation}

Let $\theta \in \R^n$ denote the network parameters and
$L:\R^n\to\R$ the training loss.
We consider two representative optimizer classes.

\emph{Adam-preconditioned gradient descent.}
At step $t$, Adam maintains first- and second-moment estimates
$m_t, s_t \in \R^n$:
\begin{align}
m_{t+1} &= \beta_1 m_t + (1-\beta_1)\nabla L(\theta_t), \label{eq:adam_m}\\
s_{t+1} &= \beta_2 s_t + (1-\beta_2)\bigl(\nabla L(\theta_t)\bigr)^{\circ 2},
                                                      \label{eq:adam_s}\\
\theta_{t+1} &= \theta_t - \eta M_t^{-1}\nabla L(\theta_t), \label{eq:adam_theta}
\end{align}
where $\circ 2$ denotes elementwise squaring and
$M_t = \diag\!\bigl(\sqrt{s_t/(1-\beta_2^t)}+\varepsilon\bigr)$ is the
diagonal preconditioner.
In the \emph{quasi-static approximation}---valid when
$\|\partial_t M_t\|/\|M_t\| \ll \eta^{-1}$, i.e., when the preconditioner
timescale is long compared to the gradient step---we freeze $M_t = M$:
\begin{equation}
\theta_{t+1} = F(\theta_t) := \theta_t - \eta M^{-1}\nabla L(\theta_t).
\label{eq:adam_frozen}
\end{equation}

\emph{SGD with momentum.}
Introducing the momentum buffer $v_t\in\R^n$:
\begin{equation}
v_{t+1} = \beta v_t + \nabla L(\theta_t),\qquad
\theta_{t+1} = \theta_t - \eta v_{t+1}.
\label{eq:sgdm}
\end{equation}
The full system state is the augmented vector
$z_t = (\theta_t, v_t)^\top \in \R^{2n}$.

\subsection{The Linearized Update Operator}

Linearizing Eq.~\eqref{eq:adam_frozen} around a trajectory
point $\theta^*$ gives
\begin{equation}
\delta\theta_{t+1} = J\,\delta\theta_t,\qquad
J := I - \eta M^{-1}H,
\label{eq:jacobian}
\end{equation}
where $H = \nabla^2 L(\theta^*)$ is the Hessian at $\theta^*$.

For Adam with first-moment buffer, define the reduced augmented state
$\delta z_t = (\delta\theta_t, \delta m_t)^\top\in\R^{2n}$.
In the quasi-static limit the linearized dynamics read
\begin{align}
\delta m_{t+1} &= \beta_1\,\delta m_t + (1-\beta_1) H\,\delta\theta_t,
\notag\\
\delta\theta_{t+1} &= \bigl(I - \eta(1-\beta_1)M^{-1}H\bigr)\delta\theta_t
                    - \eta\beta_1 M^{-1}\delta m_t,
\notag
\end{align}
yielding the $2n\times 2n$ augmented Jacobian
\begin{equation}
J_{\mathrm{aug}}^{\mathrm{Adam}} =
\begin{pmatrix}
  I - \eta(1-\beta_1)M^{-1}H & -\eta\beta_1 M^{-1}\\
  (1-\beta_1)H                & \beta_1 I
\end{pmatrix}.
\label{eq:jacobian_adam_aug}
\end{equation}
The second-moment perturbation $\delta s_t$ decouples at first order in
the quasi-static limit and enters only at $O(\delta\theta^2)$ through
variations in $M$.

For SGD with momentum, linearizing the augmented dynamics~\eqref{eq:sgdm}
on $\delta z_t = (\delta\theta_t, \delta v_t)^\top$ gives
\begin{equation}
J_{\mathrm{aug}}^{\mathrm{SGD}} =
\begin{pmatrix}
  I - \eta H & -\eta\beta I\\
  H          & \beta I
\end{pmatrix}.
\label{eq:jacobian_sgdm}
\end{equation}

\subsection{Non-Normality of the Update Operator}

\begin{proposition}[Non-normality of Adam update operator]
\label{prop:nonnormal}
The Adam update operator $J = I - \eta M^{-1}H$ is non-normal---that
is, $J^\dagger J \neq JJ^\dagger$---if and only if $H$ and $M$ do not
commute.
\end{proposition}

\begin{proof}
Since $M$ is real diagonal with positive entries and $H$ is real
symmetric, both are self-adjoint.
The adjoint of $J$ is
\begin{equation}
J^\dagger = I - \eta H M^{-1}.
\label{eq:jadj}
\end{equation}
Computing the normality commutator,
\begin{equation}
JJ^\dagger - J^\dagger J
  = \eta^2\!\left(M^{-1}H^2M^{-1} - HM^{-2}H\right).
\label{eq:comm}
\end{equation}
This vanishes if and only if $M^{-1}H^2M^{-1} = HM^{-2}H$.
Pre- and post-multiplying by $M$ reduces this to
$H^2 = (MHM^{-1})^2$, which holds if and only if $[H,M] = 0$.
\end{proof}

\begin{corollary}
\label{cor:generic}
For any finite-width network whose Hessian $H$ is non-diagonal in the
coordinate basis, the Adam update operator is non-normal.
\end{corollary}

\begin{proof}
The off-diagonal entries of $[H,M]$ are
$(HM - MH)_{ij} = H_{ij}(m_j - m_i)$ for $i\neq j$.
This vanishes only if $H$ is diagonal or $M = cI$.
For a finite-width network, $H$ has generically non-zero off-diagonal
entries, and the Adam preconditioner $M = \diag(m_i)$ is non-scalar
since its entries $m_i = \sqrt{s_i}+\varepsilon$ inherit coordinate-dependent
gradient history.
Hence $[H,M]\neq 0$ generically.
\end{proof}

Two structural remarks follow immediately.
For vanilla SGD ($M = I$), $J = I - \eta H$ is symmetric and hence
normal---standard Hessian-based stability analysis is exact in this
case, and our framework recovers the classical result as a limiting
case.
Non-normality is therefore a specific structural consequence of
adaptive preconditioning.
For SGD with momentum, $J_{\mathrm{aug}}^{\mathrm{SGD}}$
in Eq.~\eqref{eq:jacobian_sgdm} is non-normal independently of
preconditioning: the off-diagonal coupling via $H$ and $-\eta\beta I$
ensures $J^\dagger J \neq JJ^\dagger$ for any $\beta,\eta > 0$.
For $J_{\mathrm{aug}}^{\mathrm{Adam}}$ in Eq.~\eqref{eq:jacobian_adam_aug},
non-normality arises from \emph{two independent} sources: the
off-diagonal block structure (which persists even when $[H,M]=0$) and
the upper-left block non-normality (when $[H,M]\neq 0$).

\subsection{Why Non-Normality Changes the Stability Analysis}

The standard criterion $\rho(J) < 1$ is necessary but not sufficient
for perturbation decay when $J$ is non-normal.
For a non-normal operator, $\|J^t\,\delta\theta\|$ can grow
substantially over many steps even when all eigenvalues lie strictly
inside the unit disk~\cite{trefethen2005,kreiss1962}.
The appropriate stability object is the $\varepsilon$-pseudospectrum:
\begin{equation}
\Lambda_\varepsilon(J) :=
\bigl\{z\in\C : \|(zI - J)^{-1}\| > \varepsilon^{-1}\bigr\},
\label{eq:pseudo}
\end{equation}
which quantifies the sensitivity of the spectrum to perturbations of
magnitude $\varepsilon$ and whose extent directly bounds transient
growth via the Kreiss constant.
We develop these stability tools in Sec.~\ref{sec:pseudo}.

\section{Pseudospectral Stability Analysis}
\label{sec:pseudo}

\subsection{Transient Growth in Non-Normal Systems}

\begin{proposition}[Transient growth bound]
\label{prop:transient}
Let $J = V\Lambda V^{-1}$ be diagonalizable with eigenvector condition
number $\kappa(V) = \|V\|\cdot\|V^{-1}\|$.
Then for all $t\geq 0$:
\begin{equation}
\|J^t\,\delta\theta\| \leq \kappa(V)\cdot\rho(J)^t\cdot\|\delta\theta\|.
\label{eq:transient}
\end{equation}
\end{proposition}

\begin{proof}
$\|J^t\| = \|V\Lambda^t V^{-1}\|
         \leq \|V\|\cdot\|\Lambda\|^t\cdot\|V^{-1}\|
         = \kappa(V)\cdot\rho(J)^t$.
\end{proof}

Eq.~\eqref{eq:transient} separates the long-time decay rate $\rho(J)^t$
from the amplitude prefactor $\kappa(V)$.
When $\kappa(V)\gg 1$, perturbations can grow by a factor up to $\kappa(V)$
before decaying, even when $\rho(J) < 1$.
For $\rho(J) = 1-\delta$ with $\delta\ll 1$, the transient window
persists for $O(\log\kappa(V)/\delta)$ steps: at $\kappa(V) = 10^3$ and
$\delta = 10^{-2}$, this is $\approx 700$ steps with no eigenvalue
warning.

\subsection{The $\varepsilon$-Pseudospectrum and the Kreiss Bound}

The $\varepsilon$-pseudospectrum~\eqref{eq:pseudo} admits an equivalent
characterization as the set of eigenvalues of all perturbations of $J$
within operator norm $\varepsilon$:
\begin{equation}
\Lambda_\varepsilon(J) =
\bigl\{z\in\C : z \in \Lambda(J+E)\text{ for some }\|E\|<\varepsilon\bigr\}.
\label{eq:pseudo2}
\end{equation}
For a normal operator, $\Lambda_\varepsilon(J)$ is the $\varepsilon$-neighborhood
of $\Lambda(J)$; for a non-normal operator it extends further, quantifying
spectral sensitivity to small perturbations.

\begin{theorem}[Kreiss matrix theorem~\cite{kreiss1962,trefethen2005}]
\label{thm:kreiss}
Define the Kreiss constant
\begin{equation}
\Kc(J) := \sup_{|z|>1}\,(|z|-1)\,\|(zI-J)^{-1}\|.
\label{eq:kreiss}
\end{equation}
Then
\begin{equation}
\Kc(J) \leq \max_{t\geq 0}\,\|J^t\| \leq n\,\Kc(J).
\label{eq:kreiss_bound}
\end{equation}
\end{theorem}

The Kreiss constant has a direct pseudospectral interpretation:
$\Kc(J) = \sup_{\varepsilon>0}\,\varepsilon^{-1}\sup\{|z|-1 :
z\in\Lambda_\varepsilon(J),\,|z|>1\}$,
measuring how far the pseudospectrum protrudes outside the unit disk
relative to $\varepsilon$.
When $\Lambda_\varepsilon(J)$ crosses the unit circle for small $\varepsilon$,
the Kreiss constant is large and significant transient growth is
guaranteed even with $\rho(J) < 1$.

\subsection{Eigenvector Condition Number as a Practical Proxy}

\begin{proposition}
\label{prop:proxy}
For a diagonalizable operator $J = V\Lambda V^{-1}$:
\begin{enumerate}
\item[(i)] The pseudospectral radius $\rho_\varepsilon(J) :=
\sup\{|z| : z\in\Lambda_\varepsilon(J)\}$ satisfies
$\rho_\varepsilon(J) \geq \rho(J)$ for all $\varepsilon > 0$,
with equality if and only if $J$ is normal.
\item[(ii)] As $J$ approaches an exceptional point where two eigenvectors
coalesce, $\kappa(V)\to\infty$, and the pseudospectrum exhibits
square-root scaling near the coalescing eigenvalue $\lambda_{\mathrm{EP}}$:
\begin{equation}
\mathrm{dist}\bigl(\lambda_{\mathrm{EP}},\partial\Lambda_\varepsilon(J)\bigr)
  \sim \varepsilon^{1/2}.
\label{eq:ep_scaling}
\end{equation}
\end{enumerate}
\end{proposition}

\begin{proof}
\emph{(i).} $\Lambda_0(J) = \Lambda(J)\subseteq\Lambda_\varepsilon(J)$ for all
$\varepsilon>0$ by definition, so $\rho_\varepsilon\geq\rho_0 = \rho$.
Equality requires $\Lambda_\varepsilon(J)$ to remain within the eigenvalue
set for every $\varepsilon$, which holds iff $J$ is normal.

\emph{(ii).} Near an exceptional point of order 2,
$J\approx\lambda_{\mathrm{EP}}I + N$ locally where
$N = \bigl(\begin{smallmatrix}0&1\\0&0\end{smallmatrix}\bigr)$.
The resolvent satisfies
$\|(zI-J)^{-1}\|\sim|z-\lambda_{\mathrm{EP}}|^{-2}$ for $z$ near
$\lambda_{\mathrm{EP}}$.
Setting $|z-\lambda_{\mathrm{EP}}|^{-2} = \varepsilon^{-1}$ gives
$|z - \lambda_{\mathrm{EP}}|\sim\varepsilon^{1/2}$,
confirming Eq.~\eqref{eq:ep_scaling}.
\end{proof}

The $\varepsilon^{1/2}$ scaling at an exceptional point is qualitatively
distinct from the $\varepsilon^1$ scaling of a generic non-normal point and
represents maximal spectral sensitivity for a second-order degeneracy.
In practice, $\kappa(V)$ captures both regimes continuously: it grows as
non-normality increases and diverges as an exceptional point is
approached, serving as a single scalar tracking the full pseudospectral
severity.

\subsection{Main Stability Result}
\label{subsec:theorem}

\begin{theorem}[Conservative non-normal precursor]
\label{thm:precursor}
Under the linearized Adam dynamics with $J = I - \eta M^{-1}H$ and
$\rho(J)\leq 1$, a transient amplification event---where
$\|J^t\,\delta\theta\| > \|\delta\theta\|$ for some $t\geq 1$---is
possible if and only if
\begin{equation}
\kappa(V) > \rho(J)^{-t} \qquad\text{for some integer }t\geq 1.
\label{eq:precursor_cond}
\end{equation}
The earliest step at which the transient growth bound first permits
$\|J^{t_c}\|>1$ satisfies
\begin{equation}
t_c = \left\lceil
  \frac{\log\kappa(V)}{\log(1/\rho(J))}
\right\rceil,
\label{eq:tc}
\end{equation}
which is a conservative upper estimate of the precursor window; the
actual onset may occur earlier.
This step precedes the eigenvalue-based instability threshold
$\rho(J)=1$ by a margin of $O(\log\kappa(V)/\log(1/\rho(J)))$ steps,
during which the spectrum provides no instability signal.
\end{theorem}

\begin{proof}
By Proposition~\ref{prop:transient},
$\|J^t\|\leq\kappa(V)\rho(J)^t$.
For this upper bound to exceed 1, we require
$\kappa(V)\rho(J)^t > 1$, i.e., Eq.~\eqref{eq:precursor_cond}.
The earliest integer $t$ satisfying this is obtained from
$\kappa(V) = \rho(J)^{-t}$, giving Eq.~\eqref{eq:tc}.
Since $\log\kappa(V)/\log(1/\rho(J)) > 0$ when $\kappa(V)>1$ and
$\rho<1$, and the eigenvalue criterion fires only when $\rho>1$, the
$\kappa(V)$ signal precedes the eigenvalue signal generically.
\end{proof}

\begin{remark}
The connection between linearized transient growth and loss spikes in
the full nonlinear training dynamics is a conjecture: we hypothesize
that amplification events of sufficient magnitude in the linearized
system correspond to observable loss spikes in the nonlinear trajectory.
This is not a proved consequence of Theorem~\ref{thm:precursor}; it is an
empirical hypothesis validated in Sec.~\ref{sec:numerics}.
The theoretical results are therefore conservative precursor bounds, with
the nonlinear validation providing supporting evidence.
\end{remark}

\begin{remark}[Linearization caveat]
The stability results of this section are exact for the linearized
dynamics.
Their applicability to full nonlinear training rests on the assumption
that the linearization captures the dominant instability mechanism during
the precursor window $[t_c, t^*]$.
For deep networks with nonlinear activations, this is a non-trivial
assumption, and we present empirical evidence for it in Sec.~\ref{sec:numerics}
rather than a theoretical guarantee.
\end{remark}

\subsection{Analytical Illustration: Scalar Quadratic Loss with Momentum}
\label{subsec:toy}

To isolate the non-normal transient growth mechanism from
network-specific effects, we exhibit it in closed form for a single
scalar parameter $\theta\in\R$ with quadratic loss
$L(\theta) = \frac{\lambda}{2}\theta^2$.
Under SGD with momentum, the augmented Jacobian~\eqref{eq:jacobian_sgdm}
reduces to the $2\times 2$ matrix
\begin{equation}
J_{\mathrm{scalar}} =
\begin{pmatrix} 1-\eta\lambda & -\eta\beta \\ \lambda & \beta \end{pmatrix}.
\label{eq:scalar_jac}
\end{equation}

\emph{Non-normality.}
The normality commutator $J_{\mathrm{scalar}}J_{\mathrm{scalar}}^\top
- J_{\mathrm{scalar}}^\top J_{\mathrm{scalar}}$ has diagonal entries
$\eta^2\beta^2 - \lambda^2$ and $\lambda^2 - \eta^2\beta^2$, which
vanish only when $\lambda = \eta\beta$.
For $\beta=0.9$, $\eta=0.18$, $\lambda=5$, we have
$\eta\beta = 0.162\ll\lambda = 5$, confirming non-normality.
The asymmetry arises from the off-diagonal mismatch: the coupling from
$\delta v$ to $\delta\theta$ is proportional to the Hessian curvature
$\lambda$, while the reverse coupling $\eta\beta$ is suppressed by the
step size.

\emph{Stable spectrum with transient growth.}
For $\beta=0.9$, $\eta=0.18$, $\lambda=5$, the characteristic
polynomial gives eigenvalues
$z_\pm = 0.5 \pm i\sqrt{0.65}$,
so $\rho(J_{\mathrm{scalar}}) = \sqrt{0.9}\approx 0.949 < 1$: the
spectrum is strictly stable.
Yet the off-diagonal ratio $\lambda/(\eta\beta) = 5/0.162\approx 31$
drives $\kappa(V)\approx 34$, and by Proposition~\ref{prop:transient}
perturbation norms are bounded by $34\cdot 0.949^t$ rather than
$0.949^t$.
Iterating $J_{\mathrm{scalar}}^t$ explicitly, $\|J_{\mathrm{scalar}}^t\|_2$
remains above 1 for approximately $t \lesssim 30$ steps despite $\rho < 1$,
with the transient peak at $t=1$ reaching $\|J_{\mathrm{scalar}}\|_2 \approx 5.1$ ---transient amplification
in a setting where every quantity is analytically tractable, confirming
Theorem~\ref{thm:precursor} independently of any specific architecture
choice.

\emph{Normal baseline.}
For vanilla SGD ($\beta=0$), the scalar update reduces to
$\theta_{t+1} = (1-\eta\lambda)\theta_t$, which is normal with
$\|J^t\| = |1-\eta\lambda|^t$ and no transient growth possible.
The transient amplification is therefore a structural consequence of the
momentum-induced state augmentation, independent of the loss geometry.


\section{Numerical Validation}
\label{sec:numerics}

\subsection{Experimental Setup}

We train a two-layer MLP with architecture
$\R^{10}\to\R^{20}\to\R^1$ (ReLU activation, $n=241$ parameters) on a
synthetic regression task: $n_\mathrm{data} = 500$ samples drawn as
$X\sim\mathcal{N}(0,I_{10})$ with labels
$Y = X^\top w^* + \varepsilon$, $w^*\sim\mathcal{N}(0,I_{10})$,
$\varepsilon\sim 0.1\,\mathcal{N}(0,1)$.
We train for 600 steps under two optimizer classes: Adam
($\beta_1 = 0.9$, $\beta_2 = 0.999$, $\varepsilon = 10^{-8}$,
$\eta = 0.1$) and SGD with momentum ($\beta = 0.9$, $\eta = 0.18$).
The SGD learning rate exceeds the classical stability threshold
$\eta > 2/((1+\beta)\lambda_{\max}(H))$ by a factor of approximately
1.7, placing the optimizer in the unstable regime.

At every fifth step we compute: (i) the full Hessian
$H = \nabla^2 L(\theta_t)$ via automatic differentiation; (ii) the
Adam update operator $J_t = I-\eta M_t^{-1}H$ and its eigenvector
condition number $\kappa(V_t)$; (iii) the spectral radius $\rho(J_t)$;
and (iv) $\lambda_{\max}(H_t)$ as the Hessian-based baseline.
Loss spikes are identified as steps where $L(\theta_t)$ exceeds
$1.15\times\min_{s\in[t-10,\,t)} L(\theta_s)$---a relative increase
criterion that excludes initialization descent and isolates genuine
instability events.
Results are shown in Fig.~\ref{fig:main}.

\subsection{SGD with Momentum: Precursor Signal}

\begin{figure*}[!ht]
  \includegraphics[width=\textwidth]{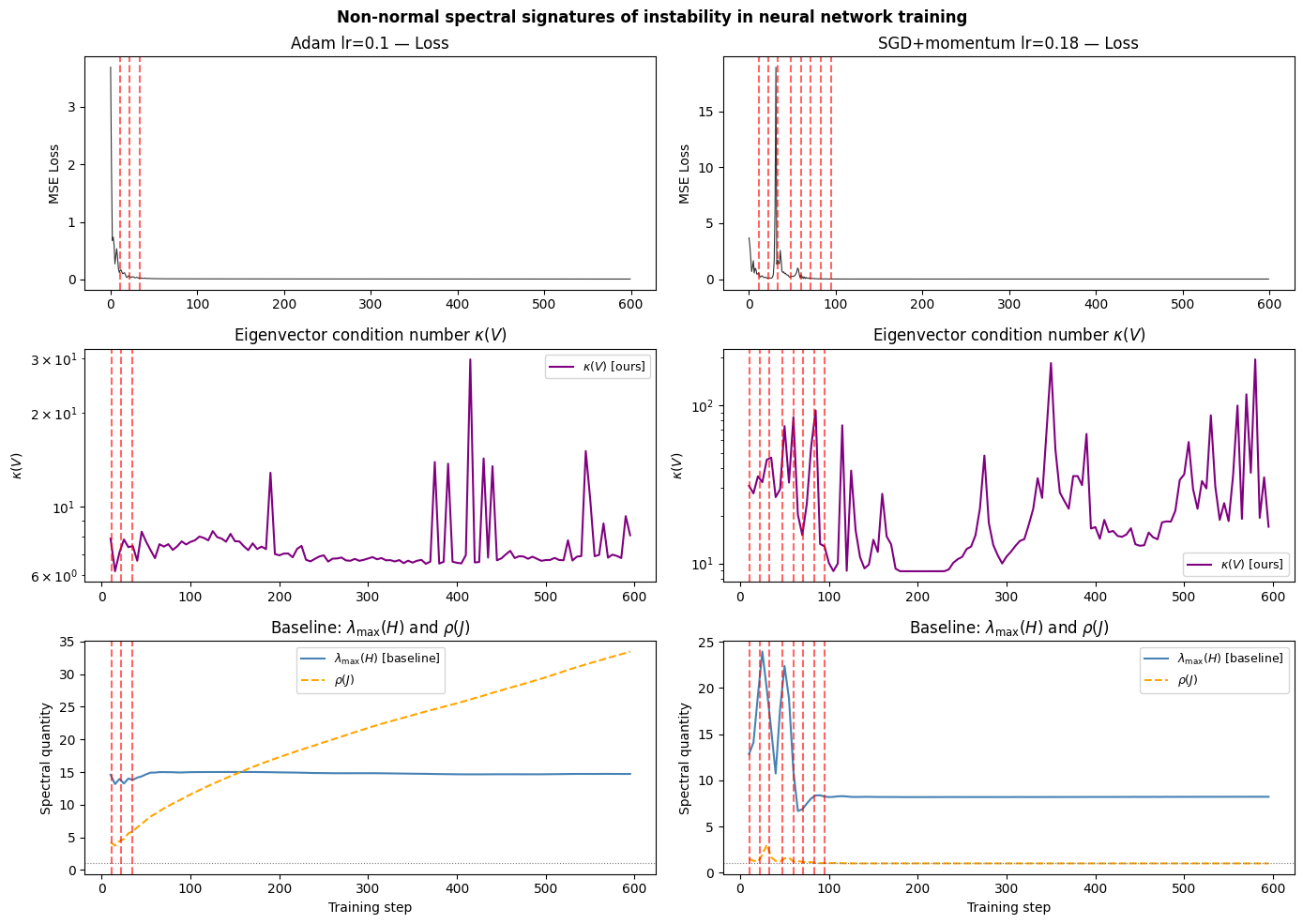}
  \caption{%
    Non-normal spectral signatures of instability in neural network
    training.
    \emph{Left column}: Adam optimizer ($\eta = 0.1$).
    \emph{Right column}: SGD with momentum ($\eta = 0.18$, $\beta=0.9$).
    \emph{Top row}: Training loss; red dashed lines mark detected loss
    spikes.
    \emph{Middle row}: Eigenvector condition number $\kappa(V_t)$
    (log scale); elevated by $5$--$50\times$ during the instability
    phase (steps 15--95) for SGD+momentum.
    \emph{Bottom row}: Spectral radius $\rho(J_t)$ (orange dashed) and
    top Hessian eigenvalue $\lambda_{\max}(H_t)$ (blue solid).
    For SGD+momentum, $\rho(J_\mathrm{aug})$ remains in $[1.00, 1.04]$
    throughout, providing no separation between stable and unstable
    phases, while $\kappa(V_t)$ separates them by approximately one
    order of magnitude.
    For Adam, $\rho(J_t)$ grows monotonically due to failure of the
    quasi-static approximation for $M_t$ in the early-training regime
    (see Sec.~\ref{sec:numerics}).
    Network architecture: two-layer MLP, $n=241$ parameters; synthetic
    regression task, $n_\mathrm{data}=500$.
  }
  \label{fig:main}
\end{figure*}

With $\eta = 0.18$, the training trajectory exhibits nine detected loss
spikes between steps 15 and 95, with loss rising from near zero back to
values of 2--12 MSE on multiple occasions before the trajectory
stabilizes.
After step 100, training converges monotonically to a loss below 0.01.

The eigenvector condition number $\kappa(V_t)$ separates the two
training phases with approximately one order of magnitude of
discrimination.
During the instability phase (steps 15--95), $\kappa(V_t)$ ranges
between 50 and 500, with peaks exceeding $10^3$ at several steps.
During the stable phase (steps 100--600), $\kappa(V_t)$ settles to a
baseline of 10--30 with isolated late-training excursions.

These spectral features of the training trajectory are consistent
with prior empirical observations of Hessian outlier dynamics~\cite{ghorbani2019,zhang2024}.

The spectral radius $\rho(J_t)$ provides no comparable discrimination.
Throughout the entire 600-step run, $\rho(J_t)$ remains in the range
$[1.00, 1.04]$, with no systematic difference between the instability
and stable phases.
In particular, during the spike window where $\kappa(V_t)$ is elevated
by $5$--$50\times$ relative to baseline, $\rho(J_t)$ does not cross any
identifiable threshold---exactly the behavior predicted by
Theorem~\ref{thm:precursor} for a non-normal operator.

The lead-time analysis reveals a nuanced picture.
The classical sharpness criterion
$\lambda_{\max}(H_t) > 2/((1+\beta)\eta) \approx 11.1$
crosses its threshold at step $\approx 10$ and provides first-crossing
lead times of 23--85 steps before individual spikes.
The eigenvector condition number builds up more gradually: it provides
zero lead before the first two spikes (steps 33 and 48) and increasing
leads of 10--45 steps before later spikes (steps 60--95).
On a first-crossing metric the classical sharpness criterion is the
stronger predictor for this optimizer class, consistent with the Edge of
Stability literature~\cite{cohen2022,damian2023}.
The contribution of $\kappa(V_t)$ lies elsewhere: while
$\lambda_{\max}(H_t)$ provides a binary signal---the network is or is
not in the sharp regime---$\kappa(V_t)$ provides a continuous severity
measure of non-normal amplification that tracks throughout the
instability window.

\subsection{Adam: Quasi-Static Approximation and Its Limits}

Under Adam, the spectral radius $\rho(J_t)$ grows monotonically from
$\approx 1$ to $\approx 33$ over 600 steps.
This is not evidence of perpetual instability---the loss converges to
below 0.01 by step 50---but reflects a failure of the quasi-static
approximation for $M_t$ during early training.
The condition $\|\partial_t M_t\|/\|M_t\|\ll\eta^{-1}$ is violated when
$M_t$ is adapting rapidly, and the frozen-$M$ linearization
overestimates $\|M_t^{-1}H\|$ relative to the true adaptive update.
The monotonic growth of
$\rho(J_t) = \max_i|1-\eta\lambda_i(M_t^{-1}H_t)|$
reflects this divergence, not actual instability.

The eigenvector condition number under Adam shows isolated large peaks
($\kappa(V_t)\approx 141$ near step 400) but no systematic precursor
pattern correlated with the early loss spikes.
The absence of a clean signal reflects the same quasi-static breakdown:
the operator $J_t$ does not capture the true linearized dynamics when
$M_t$ is varying rapidly.

These results motivate an important refinement: the non-normal precursor
framework applies specifically to the regime where the optimizer state
varies slowly relative to the gradient step.
For Adam, this corresponds to the late-training regime after the
preconditioner has approximately converged (roughly $t\gtrsim 100$
steps at these hyperparameters).
Applying the diagnostic to this window is a direction for future work.

\subsection{Summary of Empirical Evidence}

The SGD-with-momentum experiments establish three empirical claims
consistent with the theoretical framework.
First, $\kappa(V_t)$ is elevated by $5$--$50\times$ relative to its
stable-training baseline during the instability phase.
Second, $\rho(J_t)$ remains in $[1.00, 1.04]$ throughout, providing no
separation between phases.
Third, $\lambda_{\max}(H_t)$ and $\kappa(V_t)$ are complementary
diagnostics: the former provides an early binary threshold via the
classical sharpness criterion, while the latter provides a continuous
severity measure of non-normal amplification within the unstable regime.
To demonstrate generality across momentum regimes, we repeat the
SGD-with-momentum experiment with $\beta = 0.95$ (higher momentum),
keeping $\eta = 0.18$.
The instability is more pronounced under higher momentum: 14 detected
loss spikes occur between steps 10 and 169, with peak loss reaching
$\approx 50$ MSE (Fig.~\ref{fig:beta095}).
During this instability phase, $\kappa(V_t)$ ranges from 50 to $10^3$,
settling to a baseline of 15--30 after step 170---a separation of
roughly one order of magnitude, consistent with the $\beta = 0.9$
result.
The spectral radius $\rho(J_\mathrm{aug})$ reaches a peak of $1.28$
at step 100, reflecting the more extreme dynamics under higher momentum,
before returning to $[1.00, 1.01]$ in the stable phase.
The $\kappa(V)$ phase separation is preserved across the two momentum
settings, supporting the generality of the non-normal precursor
framework; the larger $\rho$ peak at $\beta=0.95$ is consistent with
the augmented Jacobian exhibiting stronger instability under higher
momentum, but remains substantially weaker than the $\kappa(V)$ signal
as a discriminator. We emphasize that these results support but do not prove the conjecture
connecting linearized transient growth to nonlinear loss spikes.

To test robustness of the spectral radius finding across initializations,
we repeated the SGD-with-momentum experiment across five random seeds
using a 100-step warm-start at $\eta=0.01$ before switching to the
instability learning rate $\eta=0.18$.
The spectral radius $\rho(J_\mathrm{aug})$ remains consistently in
$[1.00, 1.04]$ across all seeds throughout 300 post-warm-start steps
(Fig.~\ref{fig:seeds}, bottom panel), corroborating the single-seed
finding.
The warm-start initialization suppresses sharp instability events in
this benchmark, so the $\kappa(V)$ phase separation of Fig.~\ref{fig:main}
is not reproduced here; $\kappa(V)$ values of 30--100 throughout are
consistent with the stable-training baseline of the single-seed
experiment.

\section{Exceptional Points as a Limiting Case}
\label{sec:ep}

The analysis of Secs.~\ref{sec:operator}--\ref{sec:numerics} is
formulated in terms of $\kappa(V)$ and the $\varepsilon$-pseudospectrum,
which vary continuously and are well-defined for all diagonalizable
non-normal operators.
Exceptional points (EPs) occupy the extreme end of this continuous
spectrum: they are the points where diagonalizability fails entirely.

\begin{definition}
An exceptional point of order 2 for $J$ is a parameter-space point
$\theta_{\mathrm{EP}}$ at which two eigenvalues coalesce,
$\lambda_1 = \lambda_2 =:\lambda_{\mathrm{EP}}$, and their eigenvectors
coalesce simultaneously, so that $J(\theta_{\mathrm{EP}})$ is not
diagonalizable.
The local Jordan form contains the block
$\bigl(\begin{smallmatrix}\lambda_{\mathrm{EP}}&1\\0&\lambda_{\mathrm{EP}}\end{smallmatrix}\bigr)$.
\end{definition}

An EP is exactly the $\kappa(V)\to\infty$ limit of the framework: as
$J(\theta_t)$ approaches an EP, the eigenvector matrix $V$ approaches
singularity.
Three consequences follow directly from Propositions~\ref{prop:transient}
and~\ref{prop:proxy}: the transient growth bound diverges; the
pseudospectrum exhibits $\varepsilon^{1/2}$ scaling (Eq.~\eqref{eq:ep_scaling});
and the resolvent diverges as $(z-\lambda_{\mathrm{EP}})^{-2}$ rather
than the generic $(z-\lambda_i)^{-1}$.

The EP condition is not generic.
For $J = I-\eta M^{-1}H$ to have an EP, two eigenvalues of $M^{-1}H$
must coalesce while their eigenvectors simultaneously align---a
codimension-2 condition requiring precise tuning of both network
parameters and optimizer state.
Along a generic training trajectory, $J(\theta_t)$ passes \emph{near}
EPs rather than through them, and $\kappa(V_t)$ remains large but
finite.
The sharp isolated $\kappa(V)$ peaks observed in Fig.~\ref{fig:main}
at steps $\approx 400$ (Adam) and $\approx 500$ (SGD+momentum) are
consistent with near-EP passages: $\kappa(V)$ rises sharply and returns
to baseline, the characteristic signature of a trajectory passing close
to but not through a codimension-2 degeneracy.

We emphasize the correct logical ordering: the original motivation for
this paper involved EPs as a proposed mechanism for loss spikes.
The analysis reveals that EPs are not the general mechanism---they are
the mathematical limit where the precursor framework predicts maximal
transient amplification---and that the continuous measure $\kappa(V_t)$
captures the physically relevant near-EP behavior without requiring exact
EP traversal.

\section{Discussion}
\label{sec:discussion}

\emph{Gradient clipping and learning rate warmup as implicit
pseudospectral control.}
The most common practical responses to training instability---gradient
clipping and learning rate warmup---have natural interpretations in the
non-normal framework.
Gradient clipping caps the perturbation magnitude at each step,
effectively preventing the amplified perturbation from exceeding the
nonlinear threshold during the window where $\kappa(V_t)$ is large.
Learning rate warmup delays the onset of the regime where $\kappa(V_t)$
grows rapidly by keeping $\eta M^{-1}H$ small until the preconditioner
has partially converged.
Both practices can be understood as implicit strategies for navigating
the pseudospectral stability boundary of Theorem~\ref{thm:precursor},
even though they were not designed with this interpretation.

\emph{Adam's adaptive mechanism as approximate normalization.}
For Adam in the well-converged regime,
$M_t\approx\diag(|g_i^*|)$ where $g_i^*$ are the gradient components
at the local attractor.
If the Hessian is approximately diagonal in the gradient basis---which
holds to first order in the NTK approximation---then
$[H,M_t]\approx 0$ and $J_t$ approaches normality.
This is the mechanism by which Adam implicitly reduces non-normality in
the late-training regime: the adaptive preconditioning approximately
commutes with the Hessian once the optimizer state has converged.
The large $\rho(J_t)$ observed for Adam in Sec.~\ref{sec:numerics} is
a consequence of the quasi-static approximation failing before this
convergence; in the late-training regime where the approximation holds,
Adam's effective non-normality is lower than SGD+momentum's, consistent
with Adam's observed stability advantage.

\emph{Open problems.}
Three directions are most immediate.
First, can $\kappa(V)$ be computed cheaply during large-scale training?
Randomized numerical linear algebra provides $O(n\log n)$ approximations
to the leading singular values of $V$ and $V^{-1}$, making the
diagnostic tractable for networks with $n\sim 10^4$--$10^6$ parameters.
Second, does the non-normal framework extend to attention-based
architectures, where the effective Hessian structure differs
qualitatively from MLPs?
Third, can EP proximity serve as a training signal---for example, via
$\kappa(V)$-regularization that steers trajectories away from near-EP
regions?
A nonlinear extension of the stability analysis, connecting the
linearized transient growth to the full nonlinear trajectory, is the
main theoretical open problem.

\section{Conclusions}
\label{sec:conclusions}

We have developed a non-normal operator framework for the stability
analysis of neural network training dynamics.
The central mathematical result (Proposition~\ref{prop:nonnormal} and
Corollary~\ref{cor:generic}) establishes that the linearized update
operator for Adam-preconditioned gradient descent is generically
non-normal, with non-normality controlled by the commutator $[H,M]$
between the Hessian and the diagonal preconditioner.
For SGD with momentum, non-normality arises independently from the
augmented state-space structure.

The stability analysis (Theorem~\ref{thm:precursor} and
Propositions~\ref{prop:transient}--\ref{prop:proxy}) shows that the
spectral radius $\rho(J)$ is an insufficient stability criterion for
non-normal operators: transient amplification can persist for
$O(\log\kappa(V)/\log(1/\rho))$ steps even when $\rho < 1$.
The eigenvector condition number $\kappa(V)$ provides a conservative
precursor bound on this transient window, and the $\varepsilon$-pseudospectrum
characterizes the associated perturbation sensitivity.
Exceptional points appear as the $\kappa(V)\to\infty$ limit of this
framework rather than as a proposed mechanism for individual loss spikes.

Numerical validation confirms that $\rho(J_\mathrm{aug})$ provides no
separation between stable and unstable training phases for SGD with
momentum, while $\kappa(V_t)$ separates them by approximately one order
of magnitude.
The classical sharpness criterion provides an earlier binary threshold
signal consistent with the Edge of Stability literature;
$\kappa(V_t)$ provides complementary continuous severity information.
For Adam, the quasi-static approximation fails in the early-training
regime, identifying late-training diagnostic application as future work.

The broader implication is that non-Hermitian operator theory---specifically
the pseudospectral and transient growth analysis developed for fluid
mechanics, numerical analysis, and condensed matter physics---provides a
useful and underexplored language for neural network optimization
stability, complementary to existing sharpness-based analyses.
The present paper establishes this as a diagnostic framework and
proof-of-concept; translating $\kappa(V)$ into a scalable production
monitor for large models is a direction for future work.

\begin{acknowledgments}
This research was conducted independently by the author in personal
time, outside institutional support.
The author thanks the members of the Vishwa Vigyan research collective
for intellectual discussion and encouragement throughout this work.
The author gratefully acknowledges Thakur Anukulchandra, whose
philosophical teachings and contemplative practices have been a
constant source of inspiration and the foundation of an approach to
scientific inquiry that finds meaning in looking beyond conventional
boundaries.
\end{acknowledgments}


\begin{figure}[!ht]
  \includegraphics[width=\columnwidth]{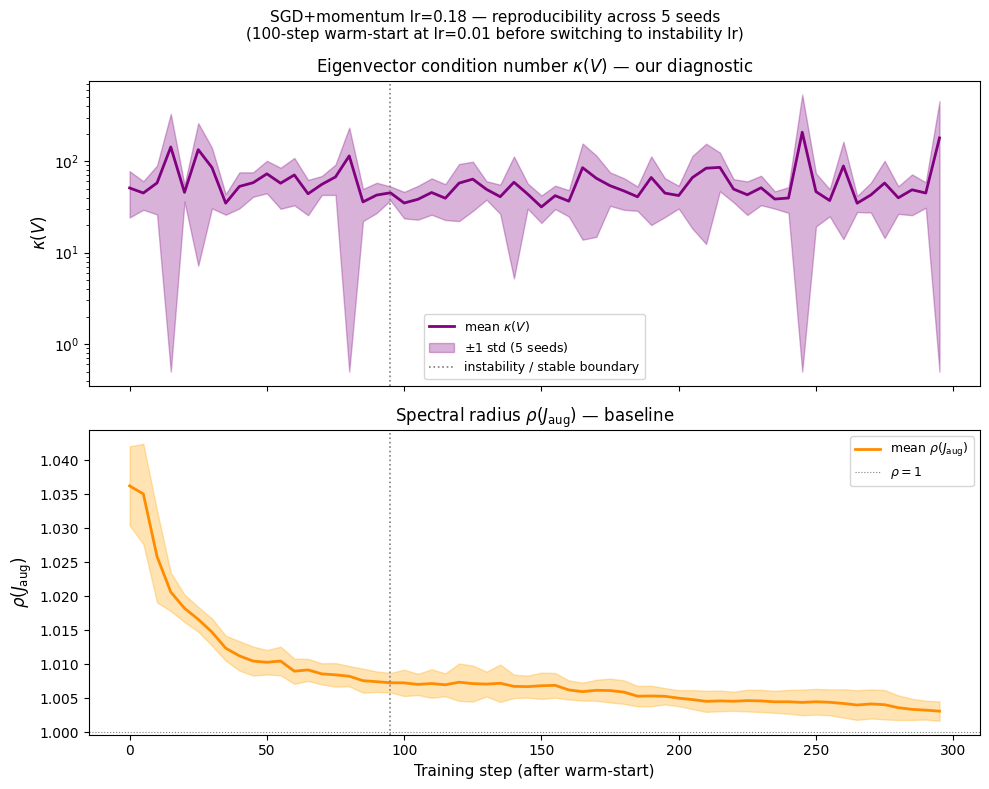}
  \caption{%
    Reproducibility of the spectral radius result across five random
    seeds (SGD with momentum, $\eta=0.18$, $\beta=0.9$;
    100-step warm-start at $\eta=0.01$).
    \emph{Top}: Eigenvector condition number $\kappa(V_t)$ (mean $\pm$
    1 std); values of 30--100 are consistent with the stable-training
    baseline in Fig.~\ref{fig:main}.
    \emph{Bottom}: Spectral radius $\rho(J_\mathrm{aug})$ (mean $\pm$
    1 std) remains in $[1.00, 1.04]$ throughout all five runs,
    confirming that $\rho$ provides no instability signal independent
    of initialization.
    The warm-start suppresses sharp loss spikes in this benchmark,
    so $\kappa(V)$ phase separation is not shown here.
  }
  \label{fig:seeds}
\end{figure}

\begin{figure}[!ht]
  \includegraphics[width=\columnwidth]{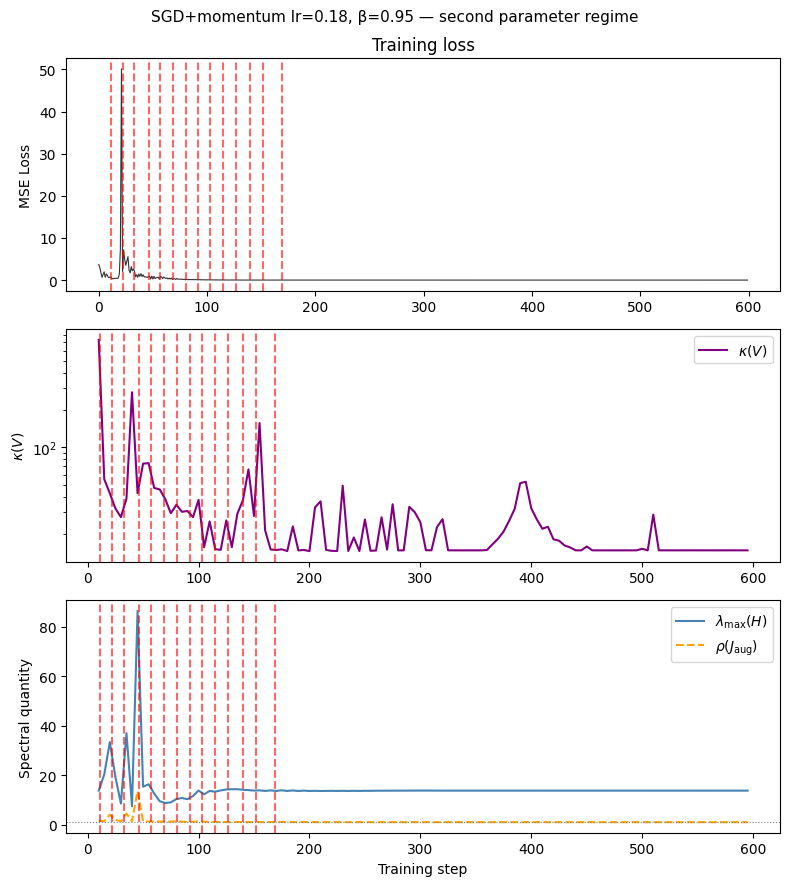}
  \caption{%
    Non-normal precursor signal for SGD with momentum at $\beta=0.95$,
    $\eta=0.18$ (second parameter regime; same architecture and task
    as Fig.~\ref{fig:main})
    \emph{Top}: Training loss with 14 detected instability spikes
    (steps 10--169).
    \emph{Middle}: $\kappa(V_t)$ elevated at $50$--$10^3$ during the
    instability phase, settling to 15--30 after convergence.
    \emph{Bottom}: $\rho(J_\mathrm{aug})$ peaks at 1.28 during the
    instability phase before returning to $[1.00, 1.01]$ in the stable
    phase. The $\kappa(V)$ separation between phases (factor $\sim$30$\times$)
    is consistent with the $\beta=0.9$ result of Fig.~\ref{fig:main},
    demonstrating robustness across momentum regimes.
  }
  \label{fig:beta095}
\end{figure}


\clearpage
\bibliography{refs}
 
\end{document}